%% file: conjformer_arxiv.tex
\documentclass{article}

\PassOptionsToPackage{numbers,compress}{natbib}
\usepackage[preprint]{neurips_2026}

\usepackage[utf8]{inputenc}
\usepackage[T1]{fontenc}
\usepackage{hyperref}
\usepackage{url}
\usepackage{booktabs}
\usepackage{amsfonts}
\usepackage{microtype}
\usepackage{xcolor}
\usepackage{amsmath,amssymb}
\usepackage{array}
\usepackage{enumitem}

\usepackage{amsmath,amsfonts,amssymb,amsthm,array}
\usepackage{caption}
\usepackage{graphicx}
\usepackage{booktabs}
\usepackage{float}
\usepackage{enumitem}
\usepackage{nicefrac}

\usepackage{enumitem}
\usepackage{algorithm}
\usepackage{algorithmic}
\input{math_commands.tex}

\input{macros}

\usepackage{graphicx}

\newcommand{\method}{\textsc{ConjFormer}\xspace}
\newcommand{\NA}{\textemdash}
\newcommand{\RETURN}{\textbf{return }}

\usepackage{tikz}

\theoremstyle{plain}

\theoremstyle{remark}

\usepackage{mdframed} 
\usepackage{thmtools}

\usepackage[flushleft]{threeparttable} 

\usepackage{caption}
\usepackage{multirow}
\usepackage{colortbl}
\definecolor{bgcolor}{rgb}{0.8,1,1}
\definecolor{bgcolor2}{rgb}{0.8,1,0.8}
\definecolor{niceblue}{rgb}{0.0,0.19,0.56}

\hypersetup{colorlinks,linkcolor={blue},citecolor={blue},urlcolor={blue}}

\usepackage{sidecap}

\definecolor{shadecolor}{gray}{0.9}
\declaretheoremstyle[
headfont=\normalfont\bfseries,
notefont=\mdseries, notebraces={(}{)},
bodyfont=\normalfont,
postheadspace=0.5em,
spaceabove=1pt,
mdframed={
  skipabove=8pt,
  skipbelow=8pt,
  hidealllines=true,
  backgroundcolor={shadecolor},
  innerleftmargin=4pt,
  innerrightmargin=4pt}
]{shaded}

\usepackage{xspace}

\title{\bf Privacy from Symmetry: Orthogonally Equivariant Transformers for LLM Inference}

\newcommand{\affmark}[1]{\textsuperscript{#1}}

\author{%
Alexander Yukhimchuk\affmark{1} \quad
Andrey Shulga \quad
Mladen Kolar\affmark{1,2} \quad
Martin Takáč\affmark{1}
\\[0.5em]
{\normalfont\small
\affmark{1} MBZUAI \quad
\affmark{2} University of Southern California
}
}

\begin{document}

\maketitle

\begin{abstract}
    Running large language models locally is often impractical, pushing inference on sensitive text to third-party providers. Split inference partially mitigates this by keeping tokens on the client and sending only hidden representations, but these representations can still be recovered via nearest-neighbor search against the public embedding table. We propose an orthogonal obfuscation procedure in which the client multiplies embeddings by a secret orthogonal matrix before transmission. To enable correct inference under arbitrary rotations, we introduce \textsc{ConjFormer}, a transformer variant that is exactly $\mathrm{O}(d)$-equivariant via a lightweight normalization change (scalar RMSNorm) together with blockwise orthogonal conjugation of all linear weights. As a result, the server performs the full forward pass entirely in the rotated basis and never observes unrotated hidden states. Experiments on GPT-2 and Llama~3.2~1B models fine-tuned on PubMed show that orthogonal obfuscation eliminates direct cosine nearest-neighbor inversion and reduces token recovery from over 35\% top-10 to at most 1.3\%, while increasing perplexity by only 0.4\% after fine-tuning. These results indicate that enforcing symmetry at the architectural level can provide a practical defense for privacy-preserving LLM inference without noise injection or heavy cryptographic machinery.
\end{abstract}

\section{Introduction}

Large language models (LLMs) have become a central paradigm in modern NLP, achieving state-of-the-art performance across a wide range of tasks~\citep{devlin2019bertpretrainingdeepbidirectional, liu2019robertarobustlyoptimizedbert, Qiu_2020, choukrani2025llm, iklassov2024self,zhumabayeva2024magbert}. As model sizes continue to grow, LLMs are increasingly deployed as cloud-based services, allowing clients to query remote models without maintaining local computational resources~\citep{bommasani2022opportunitiesrisksfoundationmodels, brown2020languagemodelsfewshotlearners}. However, this deployment model raises a critical privacy concern: during inference, users must send their data to the model provider, which can expose sensitive or proprietary information~\citep{chi2018privacypartitioningprotectinguser}. Ensuring privacy of user data at inference time therefore remains an essential and still challenging problem.

Most privacy-preserving methods focus on training-time protection, often using differential privacy and related techniques~\citep{Abadi_2016}. In contrast, \emph{inference-time privacy} for LLMs has received comparatively less attention. Cryptographic approaches such as homomorphic encryption or secure multiparty computation offer strong guarantees, but they introduce prohibitively high computational overhead for practical autoregressive inference~\citep{frikha2025obfuscatuneobfuscatedoffsitefinetuning, xiao2023offsitetuningtransferlearningmodel}.

Another systems-level direction is to run inference inside a trusted execution environment (TEE) backed by a hardware root of trust (HRoT), using remote attestation to terminate encrypted client communication inside the enclave and thus protect model \emph{inputs and outputs} from the cloud operator~\citep{tramer2018slalom, li2024sok}. However, these guarantees rely on specific hardware and a substantial trusted software stack, and known implementation flaws or side channels through timing and resource usage may still reveal sensitive details~\citep{li2024sok}.

A widely used architecture to reduce client-side exposure is \emph{split inference} (SI), in which the client runs the embedding layer (and possibly a few early transformer blocks) locally and sends only intermediate hidden representations to the server~\citep{splitnn_gupta2018distributedlearningdeepneural, thapa2022splitfedfederatedlearningmeets, vepakomma2018splitlearninghealthdistributed, chi2018privacypartitioningprotectinguser}. Although SI avoids sending plaintext tokens, it does \emph{not} prevent reconstruction of the input.

In particular, dense text embeddings and intermediate representations are known to leak substantial information about the original text and are often invertible~\citep{song2020informationleakageembeddingmodels, morris2023textembeddingsrevealalmost, li2023sentenceembeddingleaksinformation, Huang_2024}.

Differentially private noise injection can provide formal guarantees against such invertibility~\citep{dp_privacy, dx_privacy}, but it typically leads to a significant degradation in model quality.

In the standard SI setting for pretrained LLMs, the embedding table is typically public, enabling a straightforward and robust attack: nearest-neighbor (NN) embedding inversion, which recovers tokens by matching server-observed representations against the public embedding table using cosine similarity.
We demonstrate empirically that this attack remains highly effective even after fine-tuning the model.

In this work, we introduce \method, a framework designed for the standard SI setting that directly targets embedding inversion by applying a secret orthogonal transformation to the client-side embeddings. To make such obfuscation feasible, our framework combines (i) lightweight architectural modifications that make each transformer block equivariant to the action of the orthogonal group \(\mathrm{O}(d)\), and (ii) blockwise orthogonal conjugation of the weights, ensuring that the model's functionality is preserved \emph{exactly} when input embeddings are rotated.
Our implementation is available at: \url{https://github.com/katcinskiy/equivariant-private-transformer}.

\subsection{Related Works}

We first review works on privacy-preserving inference, then focus on equivariant and invariant architectures as potential solutions.

\paragraph{Cryptographic approaches.}

A classical approach to private inference relies on secure multiparty computation (MPC) and homomorphic encryption (HE)~\citep{frikha2025obfuscatuneobfuscatedoffsitefinetuning, xiao2023offsitetuningtransferlearningmodel}.
Although these methods provide strong cryptographic guarantees, they are very slow for real-world applications: HE introduces a significant overhead on every matrix multiplication and nonlinear operation~\citep{brutzkus2019lowlatencyprivacypreserving}. 
As a result, pure cryptographic approaches remain impractical for LLM inference.

\paragraph{Prompt modifications.}

Another line of work protects sensitive information by modifying the prompt before it is sent to the server. For example, this may involve masking or replacing identifiable entities~\citep{feyisetan2019privacyutilitypreservingtextualanalysis}, or rewriting or abstracting the input~\citep{yan2024protectingdataprivacylarge, west2019bottlesumunsupervisedselfsupervisedsentence}.
However, such approaches provide limited protection, as masked content can often be inferred from context using the model's prior knowledge~\citep{vats2023recoveringprivacypreservingmaskinglarge}.

\paragraph{Differentially private embedding obfuscation.}

These methods operate in an SI setting and apply the differential privacy (DP)~\citep{dp_privacy} mechanism to hidden representations before sending them to the server~\citep{Abadi_2016, Du_2023, roberts2025learningobfuscationsllmembedding, mai2024splitanddenoiseprotectlargelanguage, feyisetan2019privacyutilitypreservingtextualanalysis}. 
For example, \citet{mai2024splitanddenoiseprotectlargelanguage} adds Laplacian noise to the input embeddings to satisfy \(d_{\chi}\)-privacy~\citep{dx_privacy} guarantees.
Although such methods provide formal privacy guarantees, they face significant practical limitations: DP noise degrades model performance, and as noted by~\citet{mai2024splitanddenoiseprotectlargelanguage}, the size of the auxiliary decoder, which is used to preserve the functionality of the model on the client side, \emph{scales in size with the underlying server-side model}. 
By contrast, \method avoids stochastic noise injection and does not require an additional learned decoder to preserve model functionality.

\paragraph{Embedding obfuscations and architectural defenses.}

A more general line of work relies on heuristic transformations of the embedding space and models rather than formal privacy guarantees, aiming to defend against specific classes of attacks~\citep{zhang2021privacyinferenceattacksdefenses}. 
For instance, \citet{roberts2025learningobfuscationsllmembedding} learn a stochastic encoder that injects structured noise into embeddings, minimizing mutual information between the original and obfuscated representations, thereby mitigating embedding inversion.

Some works attempt to make architectural changes to the model \citep{mishra2023sentinellmsencryptedinputadaptation}.
The main problem of such methods is the uncertainty of new attacks that could be applied to them. 
For example, \citet{lin-etal-2024-inversion} show that the obfuscation from~\citet{mishra2023sentinellmsencryptedinputadaptation} can be easily reversed by exploiting the invariance of differences between neighboring components of the embedding vector.
In a related theoretical result, \citet{semenov2024justsimpletransformationdata} exploited orthogonal matrix obfuscation to show that in a simple linear regression setup, the client's data cannot be recovered without knowing its prior distribution.

Compared to prior work, our approach enforces resistance to embedding inversion at the architectural level.
While previous methods rely on noise injection or heuristic transformations of the embedding space, we make the transformer itself exactly $\mathrm{O}(d)$-equivariant.

This allows embeddings to be obfuscated by secret orthogonal transformations, with weight conjugation preserving functionality. The construction removes direct nearest-neighbor embedding inversion in the original embedding coordinates.

\paragraph{Equivariance and invariance.}

Equivariance and invariance to rotations are well studied in computer vision and scientific machine learning, where data often exhibit geometric structure~\citep{cohen2016steerablecnns, fuchs2020se3transformers3drototranslationequivariant}.
However, these methods typically target low-dimensional groups (e.g., planar or $\mathrm{SO}(3)$ rotations), whereas our setting requires equivariance to the full orthogonal group $\mathrm{O}(d)$ acting on high-dimensional embeddings.
Standard transformers are not $\mathrm{O}(d)$-equivariant; our method provides a construction that enforces this property via lightweight architectural modifications and weight conjugation.

\subsection{Main Contributions}
    
Our main contributions are as follows:
\begin{enumerate}[leftmargin=1.5em]
    
    \item We develop a transformer architecture that is \emph{exactly} \(\mathrm{O}(d)\)-equivariant through lightweight modifications (scalar RMSNorm), making it compatible with arbitrary orthogonal transformation of the embedding space.
    \item Building on this, we propose the \method framework, which applies blockwise orthogonal conjugation to all linear weights so that the model remains functionally identical under secret orthogonal obfuscation. 
    This ensures that the server only observes rotated hidden states and removes direct nearest-neighbor embedding inversion in the original embedding coordinates.
    \item We show that \method fundamentally changes the nature of attacks: once direct matching in the embedding space is eliminated, the natural next step for the adversary is to recover the hidden orthogonal transformations by aligning the model parameters after they have been conjugated.
    Building on this observation, we introduce and analyze a class of weight alignment attacks tailored to orthogonally obfuscated models, including two-sided ALS--Procrustes and Gram-based weight alignment methods.
    \item Through extensive experiments on multiple model families, including GPT-2 (Small, Medium, Large)~\citep{radford2019language} and Llama~3.2~1B~\citep{grattafiori2024llama3herdmodels}, we show that \method preserves standard training and fine-tuning behavior while substantially reducing attack effectiveness.
    In particular, embedding inversion remains highly effective on the original models even after fine-tuning, whereas under our framework, the strongest remaining weight alignment attacks achieve near-random token-recovery accuracy.
    
\end{enumerate}

To the best of our knowledge, \method is the first transformer architecture that explicitly enforces \(\mathrm{O}(d)\)-equivariance as a mechanism for privacy-preserving inference.

\subsection{Threat Model and System Overview}

\begin{figure*}[t]
\centering
\includegraphics[width=\textwidth]{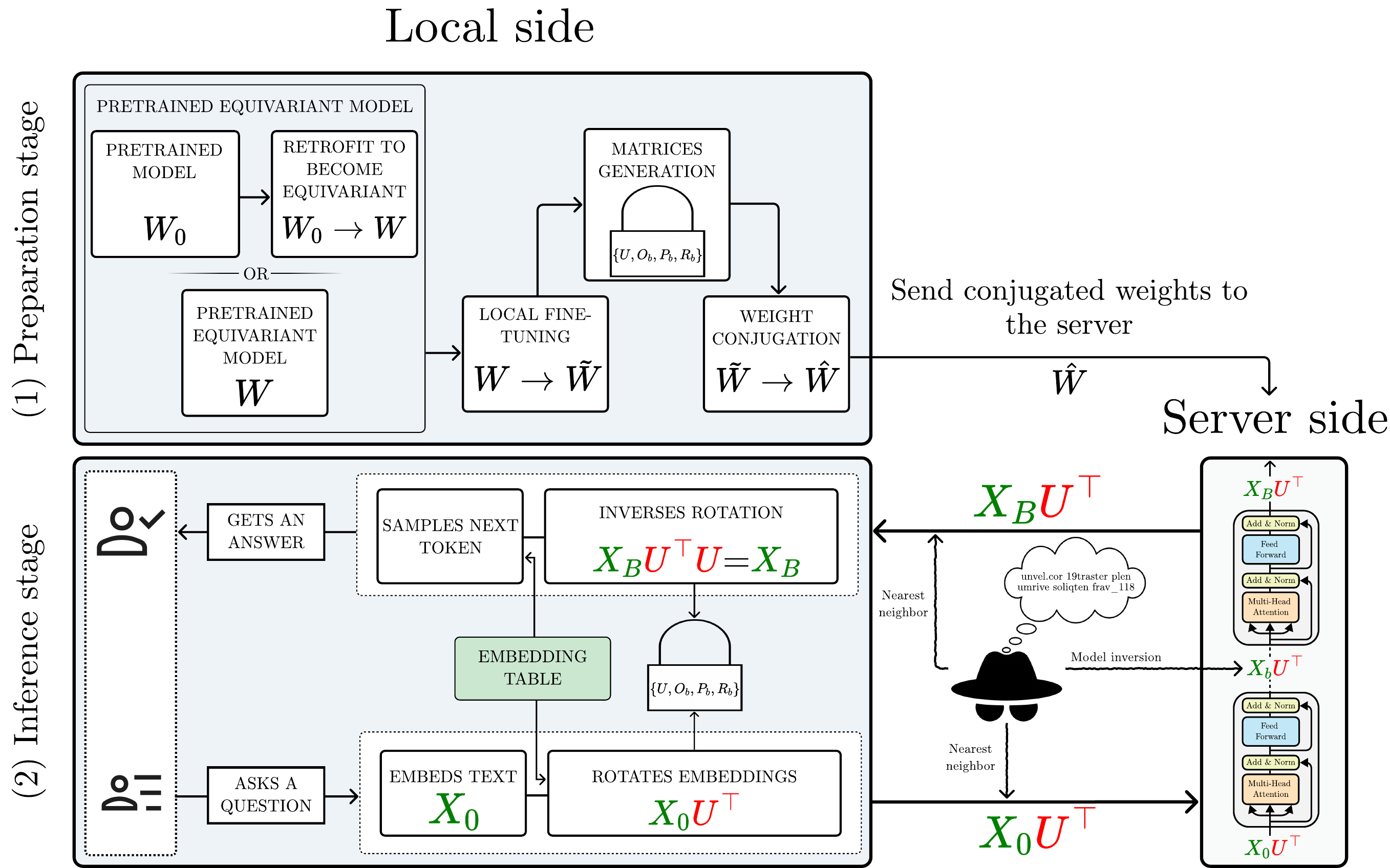}
\caption{
High-level client--server inference with \method.
During preparation, the client starts from a publicly available pretrained model (either already $\mathrm{O}(d)$-equivariant, or retrofitted to become equivariant) with weights \(W\), fine-tunes it locally to obtain \(\tilde{W}\), samples secret matrices $\{U, O_b, P_b, R_b\}$, and applies the corresponding weight conjugation (Appendix~\ref{app:conjugation}) to obtain a model \(\hat{W}\) that is functionally identical in the rotated basis.
The client then uploads the conjugated weights to the server while keeping the embedding table after fine-tuning.
During inference, the client embeds the prompt, rotates the embeddings with matrix \(U\), and sends obfuscated embeddings to the server; the server performs all computation in the rotated basis and returns the rotated output (e.g., $X_B \, U^\top$), which the client de-rotates to sample the next token via LM head locally.
}
\label{fig:high-level-architecture}

\end{figure*}

We assume an honest-but-curious server: the server follows the inference protocol but attempts to recover the client's input tokens from all observations available during inference.
Our setting is open-weight and client-controlled: the base model is public, and the server knows the architecture, public pretrained weights, uploaded conjugated transformer-body weights, and all rotated hidden states sent during inference.
The method does not aim to hide a proprietary server-side model.
The client may fine-tune locally and keeps the secret orthogonal transformations, private fine-tuning data, unrotated fine-tuned body weights, fine-tuned embedding table, and LM head private.
The adversary's goal is token recovery.
The secret key is client- or deployment-specific; if a key is disclosed, all sessions protected by that key are compromised.
Figure~\ref{fig:high-level-architecture} summarizes the setup.

\paragraph{Client-side computation.}
The client performs a limited amount of local computation.
Offline, the client stores the fine-tuned embedding table and LM head, samples the secret orthogonal transformations, conjugates the transformer-body weights once, and uploads the conjugated body to the server.
When input and output embeddings are tied~\citep{press2017usingoutputembeddingimprove}, the fine-tuned embedding table and LM head refer to the same parameter matrix.
During inference, the client computes token embeddings locally, applies the secret rotation before transmission, de-rotates the returned hidden state, and applies the local LM head for next-token prediction.
The transformer body, which accounts for most of the autoregressive inference cost, is evaluated by the server.
Therefore, the method targets clients that can perform local embedding/LM head computation and one-time preprocessing, but do not want to run the full transformer body locally.

\paragraph{Scope and limitations.}
\method is not a cryptographic or differentially private mechanism. It relies on key secrecy and on keeping the fine-tuned embedding table and LM head private on the client side. Orthogonal invariants such as norms, pairwise distances, repeated-token patterns, and attention logits remain visible to the server, so key reuse or known-plaintext observations may enable stronger attacks than those evaluated here.

\section{Methodology}

\subsection{Preliminaries}

In this section, we introduce the theoretical foundations of our work. 
Our construction relies on the action of the orthogonal group, equivariance, and the behavior of linear layers under two-sided orthogonal conjugation.

\paragraph{Orthogonal group.}
Let $\mathrm{O}(d)$ denote the group of $d \times d$ orthogonal matrices:
\[
\mathrm{O}(d) = \{ U \in \mathbb{R}^{d \times d} \mid U^\top U = I \},
\]
equipped with matrix multiplication.
Each $U \in \mathrm{O}(d)$ preserves Euclidean norms and inner products of vectors. 
In our setting, we apply an orthogonal transformation to obfuscate the model's input embeddings and conjugate the weights to preserve the model's functionality.

\paragraph{Group equivariance.}

Let a group \(G\) act on the representation space \(\mathcal{X}\) through a transformation \(\rho(g)\). 
A function \(F\) is equivariant to a group action \(\rho(g)\) if
\[
F(X\rho(g)) = F(X)\rho(g) \qquad \forall\, g \in G,\; X \in \mathcal{X}.
\]
In our construction, the relevant property is functional equivalence under simultaneous rotation of hidden states and conjugation of weights:
\[
\hat F(XU^\top)=F(X)U^\top,
\]
where \(\hat F\) denotes the conjugated block.
Thus the server evaluates the same computation in a rotated hidden-state basis.

\paragraph{Two-sided linear conjugation.}

Consider a linear transformation implemented as a PyTorch-style layer~\citep{paszke2019pytorchimperativestylehighperformance}:
\[
\mathrm{Linear}(X; W, b) = X W^\top + b,
\]
where $W \in \mathbb{R}^{d_{\mathrm{out}} \times d_{\mathrm{in}}}, \, b \in \mathbb{R}^{d_{\mathrm{out}}}$.

Given orthogonal matrices $U \in \mathrm{O}(d_{\mathrm{in}}), \, O \in \mathrm{O}(d_{\mathrm{out}})$, we define the two-sided conjugation
of a linear layer $(W, b)$ by
\[
    \hat{W} = O W U^\top, 
    \qquad 
    \hat{b} = b\, O^\top.
\]

This conjugation plays a central role in \method.
If the input is expressed in a rotated basis as \(X U^\top\), then applying the
conjugate linear layer yields a correspondingly rotated output:
\[
\begin{aligned}
    X U^\top \hat{W}^\top + \hat{b}
    &= X U^\top (O W U^\top)^\top + b\, O^\top \\
    &= X W^\top O^\top + b\, O^\top \\
    &= (X W^\top + b)\, O^\top.
\end{aligned}
\]

In other words, the linear layer performs the same computation as before, except
that its output is expressed in the rotated basis \(O^\top\).

\paragraph{Transformer components.}

A transformer block consists of (i) multi-head self-attention implemented using the linear projections \(W_Q\), \(W_K\), \(W_V\), and \(W_O\); (ii) a feed-forward network (FFN) with two linear layers and a nonlinear activation in between; (iii) residual connections; and (iv) normalization layers.  
Throughout this work, we use the standard \emph{pre-norm} transformer design~\citep{transformerswithouttears}.  
In the following sections, we describe how each of these components is adapted to ensure full \(\mathrm{O}(d)\)-equivariance of the entire transformer block.

\subsection{Architecture}

We first establish exact $\mathrm{O}(d)$-equivariance for GPT-2--style transformers, and then extend the construction to modern architectures such as Llama with SwiGLU MLPs~\citep{shazeer2020gluvariantsimprovetransformer} and RoPE~\citep{su2023roformerenhancedtransformerrotary} positional embeddings; see Appendices~\ref{app:equivariance} and~\ref{app:modern_equivariance} for details.

\method builds on a standard transformer architecture~\citep{vaswani2023attentionneed} with minimal modifications that enforce exact equivariance to the orthogonal group \(\mathrm{O}(d)\).
The key design objective is the following block-level functional equivalence:
\[
\hat F_b(XU^\top)=F_b(X)U^\top,
\]
so that the entire forward pass can be carried out in an arbitrarily rotated basis.
This property allows the server to operate exclusively on obfuscated representations without ever observing unrotated hidden states.

\paragraph{Architectural equivariance.}
Achieving exact \(\mathrm{O}(d)\)-equivariance requires only a single architectural modification: replacing standard normalization with RMSNorm with a \emph{scalar} gain.
This choice ensures that normalization commutes with orthogonal transformations, while preserving the training and optimization behavior of existing transformer models.

For GPT-2 architectures, we replace LayerNorm; for the Llama model, which already uses RMSNorm, we replace the learnable gain vector with a scalar.
A formal proof of equivariance is given in Lemma~\ref{lem:rms-equiv}.

Note that Dropout~\citep{hinton2012improvingneuralnetworkspreventing} does not affect \(\mathrm{O}(d)\)-equivariance of the model, because it is deterministic at inference.

\paragraph{Weight conjugation.}
To preserve model functionality under orthogonal obfuscation of the input embeddings, all linear layers are transformed via two-sided linear conjugation.
As a result, rotating the input embeddings induces a corresponding rotation of all intermediate representations while leaving the underlying computation unchanged.
The exact conjugation rules are provided in Appendix~\ref{app:conjugation}.

\paragraph{Positional encoding.}

Positional encodings do not require architectural changes, but must be compatible with \(\mathrm{O}(d)\)-equivariance.
Absolute positional embeddings are fully supported, as they are added on the client side prior to obfuscation.
Relative positional encodings that modify only attention logits (e.g., additive bias schemes such as ALiBi~\citep{press2022trainshorttestlong}) are also compatible, since attention logits are invariant under orthogonal transformations.
Rotary positional embeddings (RoPE) can be supported under a restriction on the orthogonal transformation; full details are provided in Appendix~\ref{app:rope-equivariance}.

\paragraph{Inference under obfuscation.}
During inference, the client sends obfuscated embeddings \(X_0 U^\top\) to the server, which processes them through the \(B\) conjugated transformer blocks and returns the final rotated output \(X_B U^\top\).
Figure~\ref{fig:low-level-architecture} illustrates the computation performed by a single \method block.

We now formalize this block-level behavior and show that each conjugated transformer block is functionally equivalent to its original counterpart, up to an orthogonal change of basis.

\begin{figure*}[t]
\centering
\includegraphics[width=1.0\linewidth]{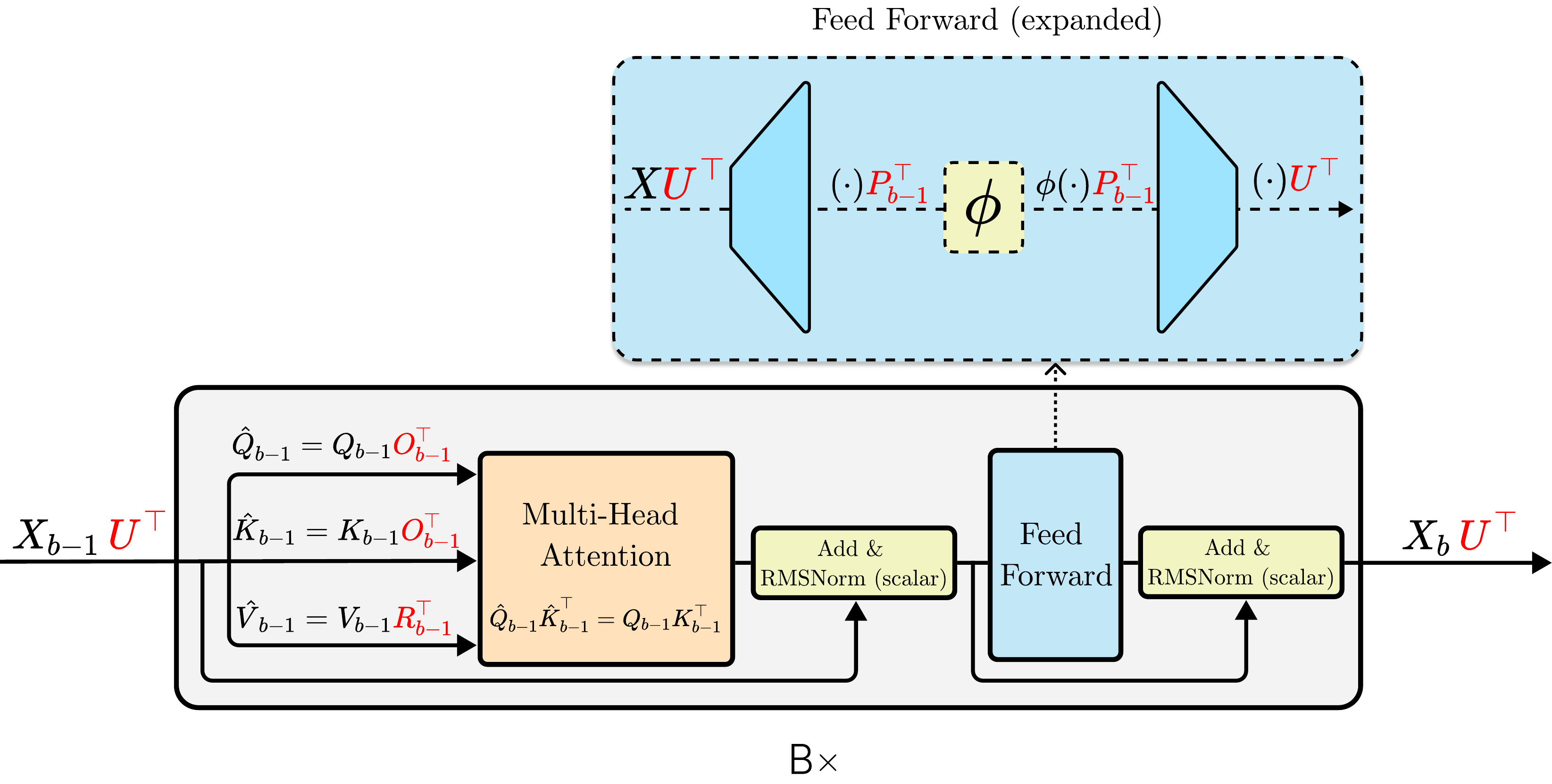}
\caption{
Low-level view of one \method block.
All server-side computations are performed in a rotated basis: the input is $X_{b - 1} U^\top$ and the output $X_{b} U^\top$.
Query/key projections are rotated by $O_b$ and values by $R_b$, which preserves attention logits:
$\hat Q_b \hat K_b^\top = Q_b K_b^\top$.
As a result, the server never observes any unrotated intermediate representations, except for the attention logits.
The feed-forward layer uses a permutation matrix $P_b$ that commutes with the nonlinearity.
}
\label{fig:low-level-architecture}

\end{figure*}

\begin{lemma}[Blockwise Functional Equivalence Under Orthogonal Conjugation]
\label{lem:functional-equivalence}
Consider a transformer block \(b\) with original parameters and its conjugate version,
whose weights and biases are defined as in Appendix~\ref{app:conjugation}.
Assume that the input to the conjugated block is expressed in the rotated basis,
\(
    \hat{X}_{b-1} = X_{b-1} U^\top.
\)
Let \(X_b\) denote the output of the original block applied to \(X_{b-1}\).
Then the output of the conjugated block satisfies
\(
    \hat{X}_{b} = X_{b} U^\top.
\)
\end{lemma}

\noindent\textit{Proof.} See Appendix~\ref{app:conjugation-proof}.

\paragraph{Normalization.}

Modern LLMs typically rely on LayerNorm~\citep{ba2016layernormalization} or RMSNorm~\citep{zhang2019rootmeansquarelayer} for normalization; however, their standard formulations are not exactly \(\mathrm{O}(d)\)-equivariant due to coordinate-dependent operations such as mean subtraction or per-dimension scaling.
We therefore replace standard normalization with RMSNorm with a scalar gain, which preserves activation scale while remaining exactly \(\mathrm{O}(d)\)-equivariant.

\begin{lemma}[$\mathrm{O}(d)$-equivariance of RMSNorm with a scalar gain]\label{lem:rms-equiv}
Let $U \in \mathrm{O}(d)$ and define
\[
\mathrm{RMSNorm}_\varepsilon(x)
= \gamma \sqrt{d}\,\frac{x}{\sqrt{\|x\|_2^2+\varepsilon}},
\qquad x \in \mathbb{R}^d,\ \varepsilon>0,
\]
with a learned scalar $\gamma \in \mathbb{R}$. Then, for all
$x \in \mathbb{R}^d$ and $U \in \mathrm{O}(d)$,
\[
\mathrm{RMSNorm}_\varepsilon(x U^\top) = \mathrm{RMSNorm}_\varepsilon(x)\,U^\top.
\]
In particular, RMSNorm with a scalar gain is exactly $\mathrm{O}(d)$-equivariant.
\end{lemma}

\noindent\textit{Proof.} See Appendix~\ref{app:rms-equivariance}.

\section{Attacks}

The main goal of \method is to remove direct nearest-neighbor embedding inversion in the original embedding coordinates. Orthogonal obfuscation does not make token recovery information-theoretically impossible; rather, it changes the attack into a basis-recovery or alignment problem.

We consider two broad families of attacks. The first family operates directly in representation space and leverages invariants preserved under orthogonal transformations. For example, since an orthogonal map preserves $\ell_2$-norms (and pairwise distances), a server observing $XU^\top$ can still access norm-based statistics of the original embeddings, which may already leak token identity.

The second family consists of \emph{weight alignment} attacks.
To maintain $\mathrm{O}(d)$-equivariance, the client applies two-sided linear conjugation to each linear layer.
Although the server does not directly observe the fine-tuned weights $\tilde W$, it has access to their conjugate versions $\hat W$ as well as the public pretrained weights $W$.
Since fine-tuning typically induces only moderate changes to pretrained weights, this residual similarity can be exploited to align $\hat W$ with $W$ and estimate the hidden orthogonal transformation $U$, thereby approximately undoing the obfuscation of the client embeddings.

\subsection{Embedding and Model Inversion Attacks}

Embedding inversion is one of the most widely used attacks in the split inference setting and is often treated as the primary privacy threat in prior work~\citep{roberts2025learningobfuscationsllmembedding, mai2024splitanddenoiseprotectlargelanguage, song2020informationleakageembeddingmodels, morris2023textembeddingsrevealalmost}. The attacker treats each hidden representation as a point in the embedding space and recovers tokens by performing a nearest-neighbor search against the public embedding table using cosine similarity. 
On an original fine-tuned model, this attack is extremely effective: in high-dimensional embedding spaces, token representations tend to be nearly orthogonal, making NN matching highly reliable~\citep{song2020informationleakageembeddingmodels, morris2023textembeddingsrevealalmost}. Mitigating this behavior requires applying a substantial global transformation to the embedding space.

More generally, model inversion attacks extend embedding inversion by learning a decoder that maps hidden representations at any layer of the model back to token sequences~\citep{Erdo_an_2022, Fredrikson2015ModelIA}. 
Such attacks rely on access to aligned pairs of hidden states and their corresponding plaintext tokens.

In \method, direct cosine nearest-neighbor inversion against the original embedding coordinates is no longer directly applicable: the server observes representations only in an unknown rotated basis and cannot simply compare them to the public embedding table.
This does not rule out stronger inversion-style attacks that first recover the basis or learn an alignment between the rotated and original spaces.

\subsection{Norm-based Invariant Attacks}

A simple class of attacks exploits quantities preserved under orthogonal transformations. Since any \(U \in \mathrm{O}(d)\) preserves Euclidean norms, a server observing \(XU^\top\) has direct access to the norm statistics of the original embeddings.

Prior work shows that token embedding norms in LLMs are nonuniform and correlate strongly with token frequency~\citep{oyama2023normwordembeddingencodes}. An attacker can therefore perform token recovery by matching $\ell_2$-norms against the public embedding table. 
Without fine-tuning the embedding matrix, this attack is nearly trivial: when token norms are distinct—as we observe empirically—recovery accuracy approaches \(100\%\). Fine-tuning the embedding table removes the norm structure exploited by this attack, as shown in Appendix~\ref{app:experiments_attacks}.

\subsection{Two-sided ALS--Procrustes Alignment}

We next consider a first \emph{weight alignment} attack that exploits the structured relationship between the public pretrained weights and the conjugate fine-tuned weights uploaded to the server.
For each transformer block $b$, the server knows the public pretrained query/key projections $(W_Q^b, W_K^b)$ and their uploaded conjugate versions $(\hat{W}_Q^b, \hat{W}_K^b)$.
These matrices are related to the unknown fine-tuned weights $(\tilde{W}_Q^b, \tilde{W}_K^b)$ via two-sided linear conjugation:
\[
\begin{aligned}
    \hat{W}_Q^b &= O_b\, \tilde{W}_Q^b\, U^\top, \\
    \hat{W}_K^b &= O_b\, \tilde{W}_K^b\, U^\top,
\end{aligned}
\qquad
U, O_b \in \mathrm{O}(d).
\]

Since fine-tuning typically induces only moderate changes, the attacker may
approximate $\tilde{W}_Q^b \approx W_Q^b$ and $\tilde{W}_K^b \approx W_K^b$ and
estimate $(U, O_b)$ by solving the two-sided orthogonal Procrustes problem:
\[
    \min_{U,\, O_b \in \mathrm{O}(d)}
    \bigl\|\hat{W}_Q^b - O_b\, W_Q^b\, U^\top\bigr\|_F^2
    \;+\;
    \bigl\|\hat{W}_K^b - O_b\, W_K^b\, U^\top\bigr\|_F^2.
\]

In contrast to the classical one-sided Procrustes problem~\citep{Schonemann_1966}, this two-sided formulation does not admit a closed-form solution.
We therefore approximate it using an alternating least-squares (ALS) procedure that alternates between estimating $O_b$ and $U$, where each update reduces to a standard one-sided Procrustes problem solved via a single SVD (see Appendix~\ref{app:als-procrustes}).

Crucially, each block is conjugated by its own orthogonal matrix $O_b$, so
alignment information cannot be aggregated across blocks in this formulation.
Therefore, the attacker can at best use the query and key projections of a single block, i.e., two matrix pairs in total.
All reported ALS--Procrustes experiments use this setting.

\subsection{Gram-based Weight Alignment}

The attacker can also exploit Gram-based weight alignment.
Again, for each block $b$, the server observes conjugate query and key projections:
\[
\hat{W}_Q^b = O_b \, \tilde{W}_Q^b \, U^\top,
\qquad
\hat{W}_K^b = O_b \, \tilde{W}_K^b \, U^\top,
\]
where $U \in \mathrm{O}(d)$ is global and $O_b \in \mathrm{O}(d)$ is block-specific.
As noted in the ALS--Procrustes attack, the block-specific transformations $O_b$ prevent direct alignment across blocks, but they cancel in the Gram matrix:
\[ 
\begin{aligned} 
&\hat{G}_b := (\hat{W}_K^b)^\top \hat{W}_Q^b = U \, (\tilde{W}_K^b)^\top \tilde{W}_Q^b \, U^\top,  \\ 
&G_b := (W_K^b)^\top W_Q^b.
\end{aligned} 
\]
As a result, the attacker obtains $B$ matrix pairs $\{(G_b,\hat{G}_b)\}_{b=1}^B$ linked by a shared unknown rotation $U$ and estimates it by solving:
\[
\min_{U \in \mathrm{O}(d)} \sum_{b=1}^B \bigl\| \hat{G}_b - U G_b U^\top \bigr\|_F^2 .
\]

We optimize this objective on $\mathrm{O}(d)$ using Riemannian gradient descent~\citep{absil2008optimization}.
Empirically, this Gram-based weight alignment attack is substantially stronger than the ALS--Procrustes alignment, since it can aggregate information across blocks by eliminating the block-specific transformation $O_b$.

\subsection{Statistical Matching: Attention Fingerprints}\label{subsec:statistical-matching}

After fine-tuning, the model's weights are fixed, and each token produces a characteristic attention logit when processed in isolation. Orthogonal obfuscation of embeddings does not affect this value, so the server could, in principle, record the logits from each head and layer produced by each obfuscated token sent by the client and treat it as a simple multi-dimensional fingerprint. With enough observations, the server may then try to match these fingerprints to those of the public vocabulary, analogous to recovering a permutation in a classical substitution cipher~\citep{shannon1949secrecy}.

These fingerprints depend on the model's fine-tuned weights, and even a small amount of additional fine-tuning on the client side changes them. This would invalidate any previously collected fingerprint statistics on the server and require the adversary to rebuild a matching table from scratch, as we show in Appendix~\ref{app:statmatch}. 

\section{Experimental Setup}

\begin{wraptable}{r}{0.48\textwidth}
\vspace{-1.2cm}
\centering
\begin{minipage}{0.48\textwidth}
    \centering
    \begin{small}
    \begin{tabular}{lccc}
        \toprule
        Model & Base & ConjFormer & $\Delta \mathrm{PPL}$ \\
        \midrule
        GPT-2 Small  & 19.77 & 20.16 & +1.97\% \\
        GPT-2 Medium & 15.75 & 15.88 & +0.83\% \\
        GPT-2 Large  & 15.63 & 15.65 & +0.13\% \\
        Llama~3.2~1B  & 15.36 & 15.40 & +0.26\% \\
        \bottomrule
    \end{tabular}
    \end{small}
    \captionof{table}{Validation perplexity on OpenWebText after pretraining the base (original) and equivariant models, showing that the equivariant architecture stays within 2\%.}
    \label{tab:pretrain-perplexity}
\end{minipage}
\vspace{-0.9cm}
\end{wraptable}
We evaluate \method on GPT-2 (Small, Medium, Large) and Llama~3.2~1B architectures in the standard language modeling setting, where the objective is to predict the next token given the previous context.

For additional details about experiments, refer to Appendix~\ref{app:experiments}. All experiments were conducted on a system equipped with $4 \times$ NVIDIA A100 GPUs.

\begin{figure*}[t]
\centering
\includegraphics[width=\textwidth]{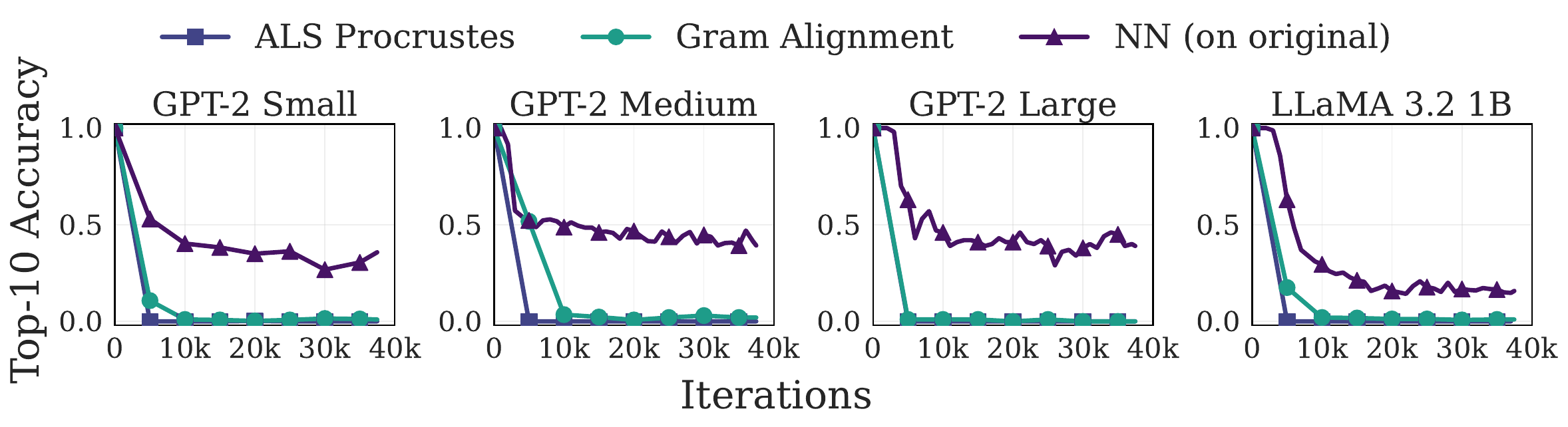}
\caption{
Top-10 token-recovery accuracy for different attacks on nonobfuscated and obfuscated models.
NN inversion is evaluated on the nonobfuscated model, while ALS--Procrustes and Gram-based weight alignment attacks are evaluated on the obfuscated model.
The Gram-based weight alignment attack is more robust than ALS--Procrustes, but still degrades to near-random performance with fine-tuning.
}
\label{fig:all-attacks}
\end{figure*}

\paragraph{Pretraining.}

Our architectural modifications affect all normalization layers and may therefore impact the pretraining process.
To evaluate this, we pretrain equivariant and original versions of models on the OpenWebText~\citep{Gokaslan2019OpenWeb} dataset for 100{,}000 iterations. Table~\ref{tab:pretrain-perplexity} reports validation perplexity with degradation below 2\%.

\paragraph{Fine-tuning.}

\begin{wraptable}{l}{0.495\textwidth}
\centering
\begin{minipage}{0.49\textwidth}
    \centering
    \begin{small}
    \begin{tabular}{lccc}
        \toprule
        Model & Base & ConjFormer & $\Delta \mathrm{PPL}$ \\
        \midrule
        GPT-2 Small  & 12.71 & 12.70 & -0.07\% \\
        GPT-2 Medium & 12.02 & 12.07 & +0.40\% \\
        GPT-2 Large  & 12.51 & 12.46 & -0.40\% \\
        Llama~3.2~1B  & 13.77 & 13.62 & -1.00\% \\
        \bottomrule
    \end{tabular}
    \end{small}
    \captionof{table}{Validation perplexity after fine-tuning on the PubMed abstracts dataset.}
    \label{tab:finetune-perf}
\end{minipage}
\end{wraptable}

We fine-tune both the baseline models and their $\mathrm{O}(d)$-equivariant variants on the PubMed abstracts dataset using identical hyperparameters for 5 epochs with 4{,}096 tokens per update.
We use a relatively large learning rate ($10^{-3}$) to ensure sufficient deviation from the pretrained weights, which is necessary to reduce the effectiveness of NN inversion and weight alignment attacks.
Table~\ref{tab:finetune-perf} reports the best validation perplexity, showing only a small performance drop for some models.

\paragraph{Retrofitting.}

The only architectural change needed for $\mathrm{O}(d)$-equivariance is replacing normalization layers with RMSNorm with a scalar gain.
To avoid pretraining from scratch, we retrofit pretrained models by making this replacement, keeping all other parameters fixed, and initializing the scalar gain to match the average scale of the original normalization.
After brief fine-tuning, all retrofitted models recover the original validation perplexity within 2{,}000 iterations; see Appendix~\ref{app:experiments_retrofitting} and Figure~\ref{fig:retrofit-perplexity-curves}.

\paragraph{Attacks.}

During the fine-tuning mentioned above, we measure how accurately an adversary can recover the original input tokens after a certain number of fine-tuning iterations.
For each position, the attacker outputs a ranked list of candidate tokens, and we report the fraction of positions where the true token appears in the top-$k$ predictions.

We evaluate the strongest applicable attack in each setting.
For the original model, we use NN embedding inversion.
For the model under \method, direct nearest-neighbor embedding inversion in the original embedding coordinates is inapplicable by construction; we therefore evaluate weight alignment attacks, including ALS--Procrustes alignment and Gram-based weight alignment.

Figure~\ref{fig:all-attacks} shows the results of the attacks during the fine-tuning process.
NN inversion on the nonobfuscated model quickly reaches a stable plateau.
In contrast, both Procrustes and Gram-based weight alignment attacks on the obfuscated model degrade with fine-tuning: the Procrustes attack collapses rapidly, while the Gram-based weight alignment—despite access to all layers—also decays to near-random performance.
We observe the same qualitative behavior across model sizes and architectures.

\section{Discussion and Limitations}

Our results show that orthogonal obfuscation, combined with an $\mathrm{O}(d)$-equivariant transformer, fundamentally changes the attacker's problem.
Without obfuscation, the adversary can directly perform embedding inversion in the embedding space, resulting in a simple, highly effective attack.
After obfuscating the hidden representations via an orthogonal rotation, direct nearest-neighbor embedding inversion in the original embedding coordinates becomes inapplicable.
Orthogonal conjugation ensures that the server-side model remains functionally equivalent under this obfuscation, but forces the adversary to first recover the hidden transformations via a substantially harder weight alignment problem.
Our approach does not provide formal cryptographic guarantees; its goal is narrower: it removes direct nearest-neighbor embedding inversion in the original embedding coordinates while preserving exact server-side inference in the rotated basis.

\newpage
\small{\bibliography{references}}


\appendix
\newpage

\part*{Supplementary Material}

\tableofcontents

\newpage

\section{Core Equivariance for GPT-2--style Transformers}
\label{app:equivariance} 
\subsection{Blockwise Conjugation Rules}
\label{app:conjugation}

In this section, we formalize weight conjugation of the transformer with pre-norm normalization that preserves its functionality while masking the secret matrix \(U\) used to transform input embeddings on the client side. 
For linear transformations, we use the PyTorch-style Linear layer:
\[
    \mathrm{Linear}(X; W, b) = XW^\top + b.
\]
We use a scalar trainable parameter in RMSNorm so that it remains $\mathrm{O}(d)$-equivariant, since the $\ell_2$-norm is invariant under orthogonal transformations.

Overall, we assume the following invariant across layers:

\begin{enumerate}
    \item Before each block, the embeddings are of the form
    \[
        \hat{X}_{\mathrm{in}} = X U^\top.
    \]
    \item After each block, the embeddings are of the form
    \[
        \hat{X}_{\mathrm{out}} = X U^\top.
    \]
\end{enumerate}

For every block \(b\), we define the following conjugated weights.
Here $U, O_b, R_b \in \mathbb{R}^{d \times d}$, where $O_b, R_b$ are block-diagonal orthogonal matrices that act independently on each attention head, $H$ is the number of attention heads and $d_h$ the per-head dimension:
\[
O_b = \mathrm{diag}(O_{b,1}, \dots, O_{b,H}), \qquad
R_b = \mathrm{diag}(R_{b,1}, \dots, R_{b,H}),
\]
where each block satisfies
\[
O_{b,h},\, R_{b,h} \in \mathrm{O}(d_h),
\]
and the full model dimension is
\[
d = H \cdot d_h.
\]

We also use an FFN mixing matrix \(P_b \in \mathbb{R}^{d_{\mathrm{FFN}} \times d_{\mathrm{FFN}}}\), which is chosen to be a permutation matrix. Thus, the conjugation rules are as follows:
\[
\begin{aligned}
    &\hat{W}_{Q}^{b} = O_{b} \tilde{W}_{Q}^{b} U^\top, \quad
     \hat{b}_{Q}^{b} = \tilde{b}_{Q}^{b} O_b^\top, \\
    &\hat{W}_{K}^{b} = O_{b} \tilde{W}_{K}^{b} U^\top, \quad
     \hat{b}_{K}^{b} = \tilde{b}_{K}^{b} O_b^\top, \\
    &\hat{W}_{V}^{b} = R_{b} \tilde{W}_{V}^{b} U^\top, \quad
     \hat{b}_{V}^{b} = \tilde{b}_{V}^{b} R_b^\top, \\
    &\hat{W}_{O}^{b} = U \tilde{W}_{O}^{b} R_b^\top, \quad
     \hat{b}_{O}^{b} = \tilde{b}_{O}^{b} U^\top, \\
    &\hat{W}_{\mathrm{FFN},1}^{b} = P_b^\top \tilde{W}_{\mathrm{FFN},1}^{b} U^\top, \quad
     \hat{b}_{\mathrm{FFN},1}^{b} = \tilde{b}_{\mathrm{FFN},1}^{b} P_b, \\
    &\hat{W}_{\mathrm{FFN},2}^{b} = U \tilde{W}_{\mathrm{FFN},2}^{b} P_b, \quad
     \hat{b}_{\mathrm{FFN},2}^{b} = \tilde{b}_{\mathrm{FFN},2}^{b} U^\top.
\end{aligned}
\]

The fact that these conjugation rules preserve the functionality of each block is formalized in Lemma~\ref{lem:functional-equivalence}, and the full proof is provided in Appendix~\ref{app:conjugation-proof}.

\subsection{Proof of Lemma~\ref{lem:functional-equivalence}}
\label{app:conjugation-proof}
\begin{proof}
Assume that the input to block \(b\) is expressed in the rotated basis,
\[
    \hat{X}_{\mathrm{in}} = X_{\mathrm{in}} U^\top.
\]

\paragraph{RMSNorm.}

Since RMSNorm uses a single scalar parameter and the \(\ell_2\)-norm is invariant under
orthogonal transformations, it holds that
\[
    \mathrm{RMSNorm}(X U^\top)
    = \mathrm{RMSNorm}(X) U^\top.
\]
A formal proof is provided in Appendix~\ref{app:rms-equivariance}.

\paragraph{Attention.}

Using the definition of the conjugated \(Q\)-projection,
\[
\begin{aligned}
    \hat{Q}
        &= \hat{X}_{\mathrm{in}} \hat{W}_Q^\top + \hat{b}_Q \\
        &= X_{\mathrm{in}}
           U^\top (O_b \tilde{W}_Q U^\top)^\top + \tilde{b}_Q\, O_b^\top \\
        &= (X_{\mathrm{in}} \tilde{W}_Q^\top + \tilde{b}_Q)\, O_b^\top
         = Q O_b^\top.
\end{aligned}
\]

For the \(K\)-projection:
\[
\begin{aligned}
    \hat{K}
        &= \hat{X}_{\mathrm{in}} \hat{W}_K^\top + \hat{b}_K \\
        &= X_{\mathrm{in}}
           U^\top (O_b \tilde{W}_K U^\top)^\top + \tilde{b}_K\, O_b^\top \\
        &= (X_{\mathrm{in}} \tilde{W}_K^\top + \tilde{b}_K)\, O_b^\top
         = K O_b^\top.
\end{aligned}
\]

Similarly, for the \(V\)-projection:
\[
\begin{aligned}
    \hat{V}
        &= \hat{X}_{\mathrm{in}} \hat{W}_V^\top + \hat{b}_V \\
        &= X_{\mathrm{in}}
           U^\top (R_b \tilde{W}_V U^\top)^\top + \tilde{b}_V\, R_b^\top \\
        &= (X_{\mathrm{in}} \tilde{W}_V^\top + \tilde{b}_V)\, R_b^\top
         = V R_b^\top.
\end{aligned}
\]

\paragraph{Attention logits.}

Because \(O_b\) is orthogonal,
\[
    \hat{Q} \hat{K}^\top
    = (Q O_b^\top)(K O_b^\top)^\top
    = Q K^\top.
\]

Since $O_b=\mathrm{diag}(O_{b,1},\dots,O_{b,H})$ acts head-wise, the identity
$\hat Q^{(h)}(\hat K^{(h)})^\top = Q^{(h)}(K^{(h)})^\top$ holds for each head $h$,
and therefore the multi-head attention logits are preserved.
Thus, the attention logits remain invariant.

\paragraph{Attention output.}

Let \(A = \mathrm{Softmax}(QK^\top / \sqrt{d_h})\, V\).
Using invariance of logits and the transformed \(V\),
\[
    \hat{A}
    = \mathrm{Softmax}\!\left(\frac{\hat{Q}\hat{K}^\top}{\sqrt{d_h}}\right)\hat{V}
    = A R_b^\top.
\]
The output projection satisfies
\[
    \hat{Y}
    = \hat{A} \hat{W}_O^\top + \hat{b}_O
    = A R_b^\top (U \tilde{W}_O R_b^\top)^\top + \tilde{b}_O U^\top
    = Y U^\top,
\]
where \(Y\) is the output of the original block.

The residual connection preserves the rotation \(U^\top\) from the right:
\[
    \hat{Z} = \hat{X}_{\mathrm{in}} + \hat{Y}
            = (X_{\mathrm{in}} + Y) U^\top
            = Z U^\top.
\]

\paragraph{FFN block.}

The FFN input satisfies \(\hat{Z} = Z U^\top\), so
\[
\begin{aligned}
    \hat{A}_{\mathrm{FFN}}
        &= \hat{Z} \hat{W}_{\mathrm{FFN},1}^\top + \hat{b}_{\mathrm{FFN},1} \\
        &= Z U^\top (P_b^\top \tilde{W}_{\mathrm{FFN},1} U^\top)^\top
           + \tilde{b}_{\mathrm{FFN},1}\, P_b \\
        &= (Z \tilde{W}_{\mathrm{FFN},1}^\top + \tilde{b}_{\mathrm{FFN},1})\, P_b
         = A_{\mathrm{FFN}} P_b.
\end{aligned}
\]
Since the FFN uses an elementwise activation function (such as ReLU or GELU), it commutes with the permutation matrix \(P_b\):
\[
    \sigma(A_{\mathrm{FFN}} P_b) = \sigma(A_{\mathrm{FFN}}) P_b.
\]
For the second linear transformation,
\[
\begin{aligned}
    \hat{Y}_{\mathrm{FFN}}
        &= \sigma(A_{\mathrm{FFN}} P_b)\hat{W}_{\mathrm{FFN},2}^\top
           + \hat{b}_{\mathrm{FFN},2} \\
        &= \sigma(A_{\mathrm{FFN}}) P_b (U \tilde{W}_{\mathrm{FFN},2} P_b)^\top
           + \tilde{b}_{\mathrm{FFN},2}\, U^\top \\
        &= (\sigma(A_{\mathrm{FFN}}) \tilde{W}_{\mathrm{FFN},2}^\top
            + \tilde{b}_{\mathrm{FFN},2})\, U^\top
         = Y_{\mathrm{FFN}} U^\top.
\end{aligned}
\]

A final residual connection yields
\[
    \hat{X}_b
    = \hat{Z} + \hat{Y}_{\mathrm{FFN}}
    = (Z + Y_{\mathrm{FFN}}) U^\top
    = X_b U^\top.
\]

\paragraph{Conclusion.}

Thus the block takes the rotated input \(X_{b-1} U^\top\) and produces the
rotated output \(X_b U^\top\), proving the lemma.
\end{proof}

\section{Extensions to Modern Architectures}\label{app:modern_equivariance}

\subsection{Proof of Lemma~\ref{lem:rms-equiv}}\label{app:rms-equivariance}

\begin{proof}
Let $x \in \mathbb{R}^d$ and let $U \in \mathrm{O}(d)$ be any orthogonal matrix.
RMSNorm with a scalar gain $\gamma$ and stability constant $\varepsilon>0$ is defined as
\[
\mathrm{RMSNorm}_\varepsilon(x)
= \gamma \sqrt{d}\,\frac{x}{\sqrt{\|x\|_2^2+\varepsilon}}.
\]
Since orthogonal matrices preserve the Euclidean norm,
\[
\|xU^\top\|_2^2 = \|x\|_2^2.
\]
Therefore,
\[
\mathrm{RMSNorm}_\varepsilon(x U^\top)
= \gamma \sqrt{d}\,\frac{xU^\top}{\sqrt{\|xU^\top\|_2^2+\varepsilon}}
= \gamma \sqrt{d}\,\frac{xU^\top}{\sqrt{\|x\|_2^2+\varepsilon}}
= \mathrm{RMSNorm}_\varepsilon(x)\,U^\top.
\]
Thus RMSNorm with a scalar gain is exactly $\mathrm{O}(d)$-equivariant.
\end{proof}

\subsection{Equivariance with RoPE}\label{app:rope-equivariance}

\paragraph{Non-equivariance of RoPE.}

We consider attention logits between positions $t$ and $s$.
Let $R(t)$ denote the RoPE operator applied on the right, so that
\[
Q_t := Q R(t), \qquad K_s := K R(s).
\]
Without obfuscation, the RoPE attention logits are
\[
\ell_{t,s}
:= Q_t K_s^\top
= Q R(t)\, (K R(s))^\top
= Q\, R(t) R(s)^\top\, K^\top.
\]

After obfuscation, we have $\hat Q = Q O_b^\top$ and $\hat K = K O_b^\top$, hence
\[
\hat Q_t := \hat Q R(t) = Q O_b^\top R(t), \qquad
\hat K_s := \hat K R(s) = K O_b^\top R(s),
\]
and the corresponding logits become
\[
\hat \ell_{t,s}
:= \hat Q_t \hat K_s^\top
= Q\, O_b^\top R(t)\, R(s)^\top O_b\, K^\top.
\]
In general, $\hat \ell_{t,s} \neq \ell_{t,s}$ because $O_b$ does not commute with $R(t)R(s)^\top$.

\paragraph{Restoring equivariance by restricting \(O_b\).}

To restore equivariance, it suffices to restrict $O_b$ to commute with $R(t)$
for all positions $t$:
\[
O_b R(t) = R(t) O_b \qquad \forall t.
\]
A simple sufficient choice is to take $O_b$ block-diagonal with aligned $2 \times 2$
rotation blocks:
\[
O_b = \bigoplus_{i=1}^{d/2} V(\theta_i),
\qquad V(\theta_i) \in \mathrm{SO}(2).
\]
Since both $O_b$ and $R(t)$ are block-diagonal with the same $2 \times 2$ block
structure, they commute, hence also $O_b$ commutes with $R(t)R(s)^\top$ for all
$t,s$.
Therefore,
\[
\hat \ell_{t,s}
= Q\, O_b^\top R(t) R(s)^\top O_b\, K^\top
= Q\, R(t) R(s)^\top\, K^\top
= \ell_{t,s},
\]
and RoPE equivariance is preserved.

\subsection{Equivariant Conjugation of SwiGLU MLPs}

Llama replaces the classical two-layer FFN with a gated MLP (SwiGLU) that lacks bias terms in its projections. For an input
$Z \in \mathbb{R}^{d}$ the MLP is
\[
\mathrm{MLP}(Z)
= \Bigl(\mathrm{SiLU}(Z \tilde{W}_{\mathrm{gate}}^\top)\ \odot\ (Z \tilde{W}_{\mathrm{up}}^\top)\Bigr)\, \tilde{W}_{\mathrm{down}}^\top,
\]
where $\tilde{W}_{\mathrm{gate}}, \tilde{W}_{\mathrm{up}} \in \mathbb{R}^{d_{\mathrm{FFN}} \times d }$,
$\tilde{W}_{\mathrm{down}} \in \mathbb{R}^{d \times d_{\mathrm{FFN}}}$, and $\odot$ denotes
elementwise multiplication.

Assume the block input is expressed in the rotated basis, $\hat{Z} = Z U^\top$.
To preserve equivariance through the gated hidden dimension, we introduce a
block-specific change of basis $P_b \in \mathbb{R}^{d_{\mathrm{FFN}} \times d_{\mathrm{FFN}}}$ in the
MLP intermediate space. 
In our implementation, $P_b$ is a permutation matrix, and therefore it commutes with elementwise operations, so that for arbitrary matrices $M$ and $N$:
\[
\sigma(N P_b) = \sigma(N) P_b,
\qquad
(N P_b)\odot(M P_b) = (N \odot M) P_b.
\]
We then update the Llama MLP weights by conjugation:
\[
\hat{W}_{\mathrm{gate}} = P_b^\top \tilde{W}_{\mathrm{gate}} U^\top,
\qquad
\hat{W}_{\mathrm{up}}   = P_b^\top \tilde{W}_{\mathrm{up}} U^\top,
\qquad
\hat{W}_{\mathrm{down}} = U \tilde{W}_{\mathrm{down}} P_b.
\]
With these definitions,
\[
\hat{Z}\hat{W}_{\mathrm{gate}}^\top = Z \tilde{W}_{\mathrm{gate}}^\top P_b,
\qquad
\hat{Z}\hat{W}_{\mathrm{up}}^\top   = Z \tilde{W}_{\mathrm{up}}^\top P_b,
\]
and therefore the gated activation transforms as
\[
\mathrm{SiLU}(\hat{Z}\hat{W}_{\mathrm{gate}}^\top)\odot(\hat{Z}\hat{W}_{\mathrm{up}}^\top)
=
\bigl(\mathrm{SiLU}(Z \tilde{W}_{\mathrm{gate}}^\top)\odot(Z \tilde{W}_{\mathrm{up}}^\top)\bigr) P_b.
\]
Finally, the down projection yields
\[
\widehat{\mathrm{MLP}}(\hat{Z})
=
\bigl(\mathrm{SiLU}(Z \tilde{W}_{\mathrm{gate}}^\top)\odot(Z \tilde{W}_{\mathrm{up}}^\top)\bigr)\,
P_b\, \hat{W}_{\mathrm{down}}^\top
=
\mathrm{MLP}(Z)\, U^\top.
\]
So the Llama MLP preserves the same rotated-basis equivariance as the rest of the block.

\paragraph{Grouped-query attention.}
In grouped-query attention, keys/values have fewer heads and are shared across groups of query heads (via simple repetition to match the query-head count).
We handle this by using the same head-wise orthogonal conjugation as in standard attention, with the only change that the rotation assigned to a KV head is reused for all query heads that share it (i.e., the same blocks are repeated across the corresponding query-head groups).

\section{Weight Alignment Attack Algorithms}

\subsection{Overview}

Under orthogonal obfuscation, direct NN inversion is no longer applicable.
Instead, the attacker must first recover the unknown orthogonal transformation $U \in \mathrm{O}(d)$ that maps obfuscated representations back to the original embedding space.
Once an estimate $\hat{U}$ is obtained, token recovery reduces to standard embedding inversion via cosine similarity against the public embedding table.

Accordingly, we consider weight alignment attacks whose primary objective is to estimate $U$ from the relationship between public pretrained weights and their conjugated counterparts.
After estimating $U$, all attacks use the same token-recovery procedure.

\subsection{ALS--Procrustes Alignment}\label{app:als-procrustes}

Because each transformer block is conjugated by its own orthogonal matrix $O_b$,
the alignment problem must be solved independently for each block.
In this appendix, we describe the ALS--Procrustes procedure for a single block
and omit the block index $b$ for simplicity.
In this setting, the attacker can use at most two matrix pairs: the query and
key projections $(W_Q,\hat W_Q)$ and $(W_K,\hat W_K)$.

The attacker estimates the hidden orthogonal matrices by solving:
\[
    \min_{O,\,U \in \mathrm{O}(d)}
    \bigl\|\hat{W}_Q - O\, W_Q\, U^\top\bigr\|_F^2
    \;+\;
    \bigl\|\hat{W}_K - O\, W_K\, U^\top\bigr\|_F^2.
\]
Since the objective is bilinear in $(O,U)$, we optimize it using an alternating
least-squares (ALS) procedure.
Each subproblem admits a closed-form solution via a single SVD.
The full procedure is summarized in Algorithm~\ref{alg:als-procrustes}.

\begin{algorithm}[H]
\caption{Two-sided ALS--Procrustes weight alignment (single block, $Q/K$)}
\label{alg:als-procrustes}
\begin{algorithmic}[1]
    \REQUIRE Public weights $(W_Q, W_K)$, conjugated weights $(\hat{W}_Q, \hat{W}_K)$, iterations $T$
    \STATE Initialize $O^{(0)} \gets I_d$, $U^{(0)} \gets I_d$
    \FOR{$t = 1,\dots,T$}
        \STATE $M_{\mathrm{out}}^{(t)} \gets
        \hat{W}_Q\, U^{(t-1)}\, W_Q^{\!\top}
        \;+\;
        \hat{W}_K\, U^{(t-1)}\, W_K^{\!\top}$
        \STATE $L, \Sigma, R^\top \gets \mathrm{SVD}\!\bigl(M_{\mathrm{out}}^{(t)}\bigr)$
        \STATE $O^{(t)} \gets L R^\top$
        \STATE $M_{\mathrm{in}}^{(t)} \gets
        \hat{W}_Q^{\!\top}\, O^{(t)}\, W_Q
        \;+\;
        \hat{W}_K^{\!\top}\, O^{(t)}\, W_K$
        \STATE $L, \Sigma, R^\top \gets \mathrm{SVD}\!\bigl(M_{\mathrm{in}}^{(t)}\bigr)$
        \STATE $U^{(t)} \gets L R^\top$
    \ENDFOR
    \STATE \RETURN $O^{(T)}, U^{(T)}$
\end{algorithmic}
\end{algorithm}

\section{Experiment Details}\label{app:experiments}

\subsection{Pretraining}

For pretraining, we used the OpenWebText dataset (2B tokens) and trained on $4 \times$ NVIDIA A100 GPUs with 65{,}536 tokens per optimizer update.
We used AdamW~\citep{loshchilov2019decoupledweightdecayregularization} ($\beta_1=0.9$, $\beta_2=0.999$, weight decay $=0.1$) with a peak learning rate of $3\times10^{-4}$ for 100{,}000 iterations, using a linear warmup over the first 2{,}000 iterations followed by cosine decay.
Figure~\ref{fig:pretrain-perplexity-curves} shows the validation perplexity during pretraining, confirming that the equivariant architecture preserves standard training behavior across all model sizes.

\begin{figure*}[ht]
    \centering
    \includegraphics[width=\textwidth]{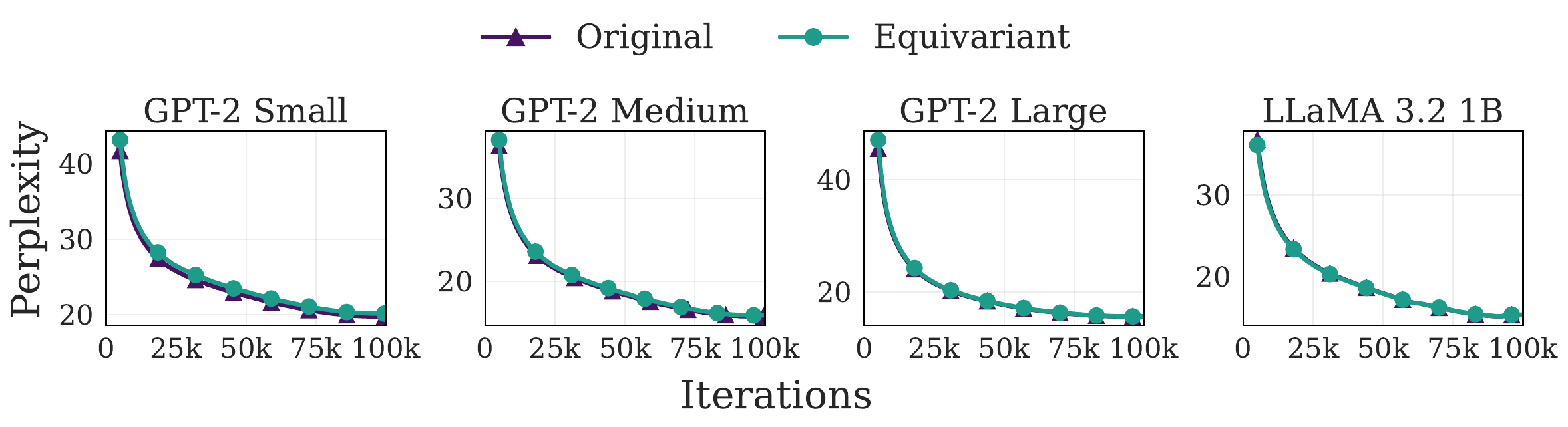}
    \caption{
    Validation perplexity during pretraining on OpenWebText for GPT-2 Small, Medium, Large, and Llama~3.2~1B.
    Equivariant models closely match original versions throughout training.
    }
    \label{fig:pretrain-perplexity-curves}
\end{figure*}

\subsection{Fine-tuning}

For fine-tuning, we trained all models on the PubMed abstracts dataset using the same optimization setup as in pretraining.
We used AdamW ($\beta_1=0.9$, $\beta_2=0.999$, weight decay $=0.1$) with a learning rate of $10^{-3}$ and 4{,}096 tokens per optimizer update, employing a linear warmup over the first 2{,}000 iterations followed by cosine decay.
Figure~\ref{fig:finetune-perplexity-curves} shows the validation perplexity during fine-tuning, indicating that the equivariant and original models exhibit similar adaptation behavior with only a minor difference in final perplexity.

\begin{figure*}[ht]
    \centering
    \includegraphics[width=\textwidth]{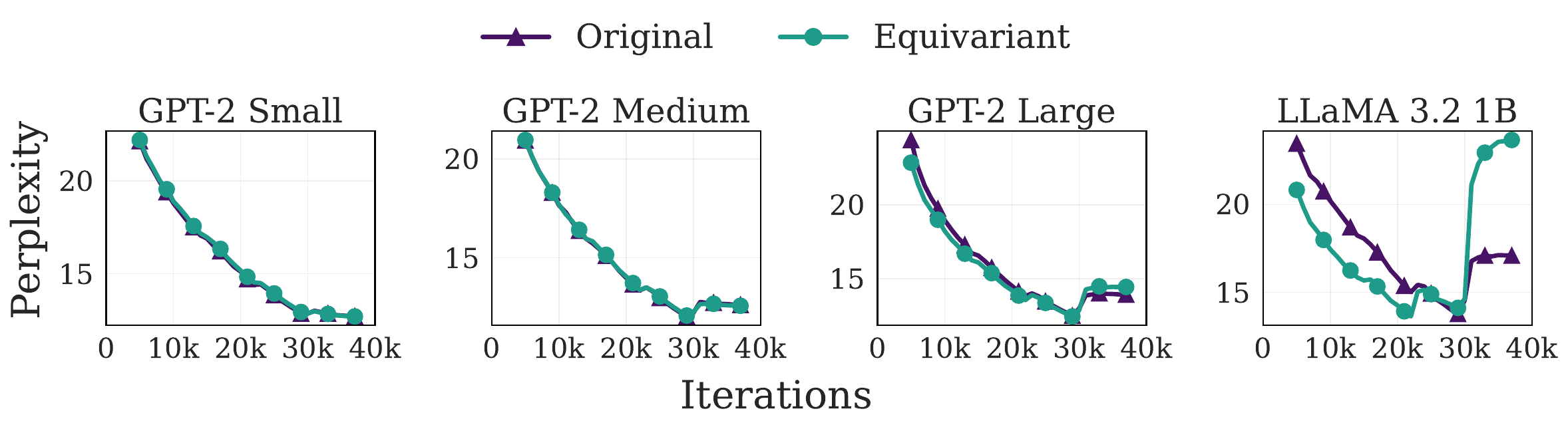}
    \caption{
    Validation perplexity during fine-tuning on the PubMed abstracts dataset for GPT-2 Small, Medium, Large, and Llama~3.2~1B.
    }
    \label{fig:finetune-perplexity-curves}
\end{figure*}

\subsection{Retrofitting}\label{app:experiments_retrofitting}

For retrofitting, we start from a publicly available pretrained model that is not $\mathrm{O}(d)$-equivariant, replace only the normalization layers with RMSNorm with a scalar gain, and then perform a short fine-tuning on the PubMed abstracts to recover performance. We use AdamW with learning rate $10^{-4}$ and 8{,}192 tokens per optimizer update. The goal is to demonstrate that \method can be obtained without retraining from scratch: equivariance can be enforced by a minimal architectural change followed by lightweight adaptation. Figure~\ref{fig:retrofit-perplexity-curves} shows the retrofitting curves.

\begin{figure*}[ht]
    \centering
    \includegraphics[width=\textwidth]{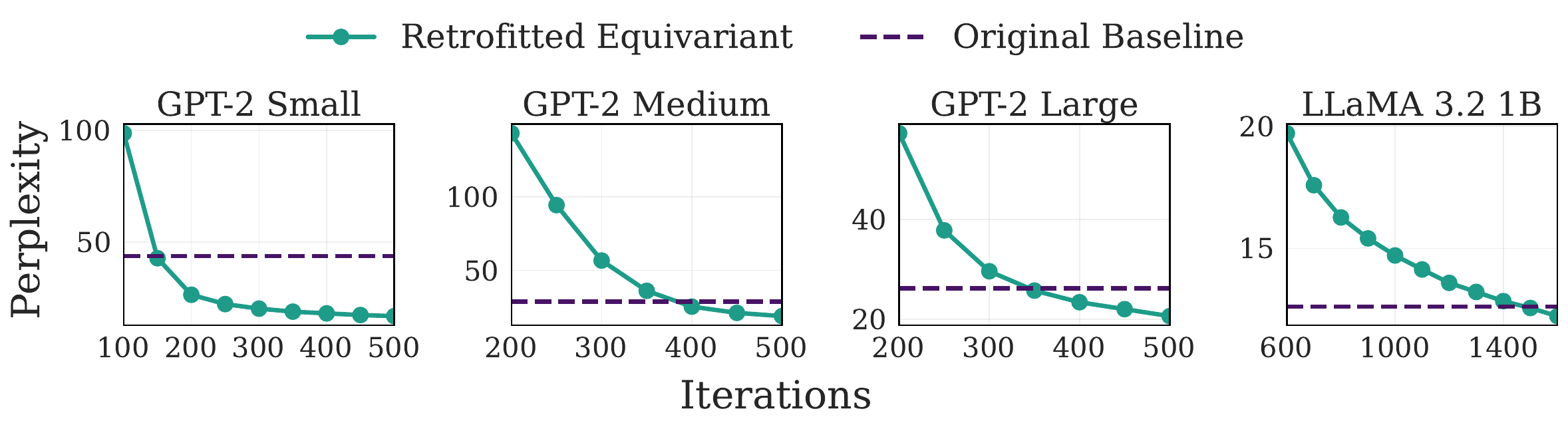}
    \caption{
    Validation perplexity during retrofitting on the PubMed abstracts dataset after replacing only the normalization layers with RMSNorm with a scalar gain.
    }
    \label{fig:retrofit-perplexity-curves}
\end{figure*}

\subsection{Attacks}\label{app:experiments_attacks}

For both the ALS--Procrustes and Gram-based weight alignment attacks, we tune the number of optimization iterations to balance reconstruction accuracy and computational efficiency. In all experiments, we use 1{,}000 iterations for the ALS--Procrustes procedure and 2{,}500 iterations of gradient descent for the Gram-based weight alignment, with learning rate $10^{-3}$. The resulting token-recovery accuracy across models and top-$k$ settings is reported in Table~\ref{tab:attacks-summary}. Additional top-1 and top-100 token-recovery curves under fine-tuning are shown in Figure~\ref{fig:attack-topk-curves}.

All reported results are averaged over multiple random seeds.
For each random seed, token-recovery accuracy is computed as the mean top-$k$ accuracy over 128 randomly sampled tokens, and the final numbers are obtained by averaging across seeds.

\begin{figure*}[ht]
    \centering
    \includegraphics[width=\textwidth]{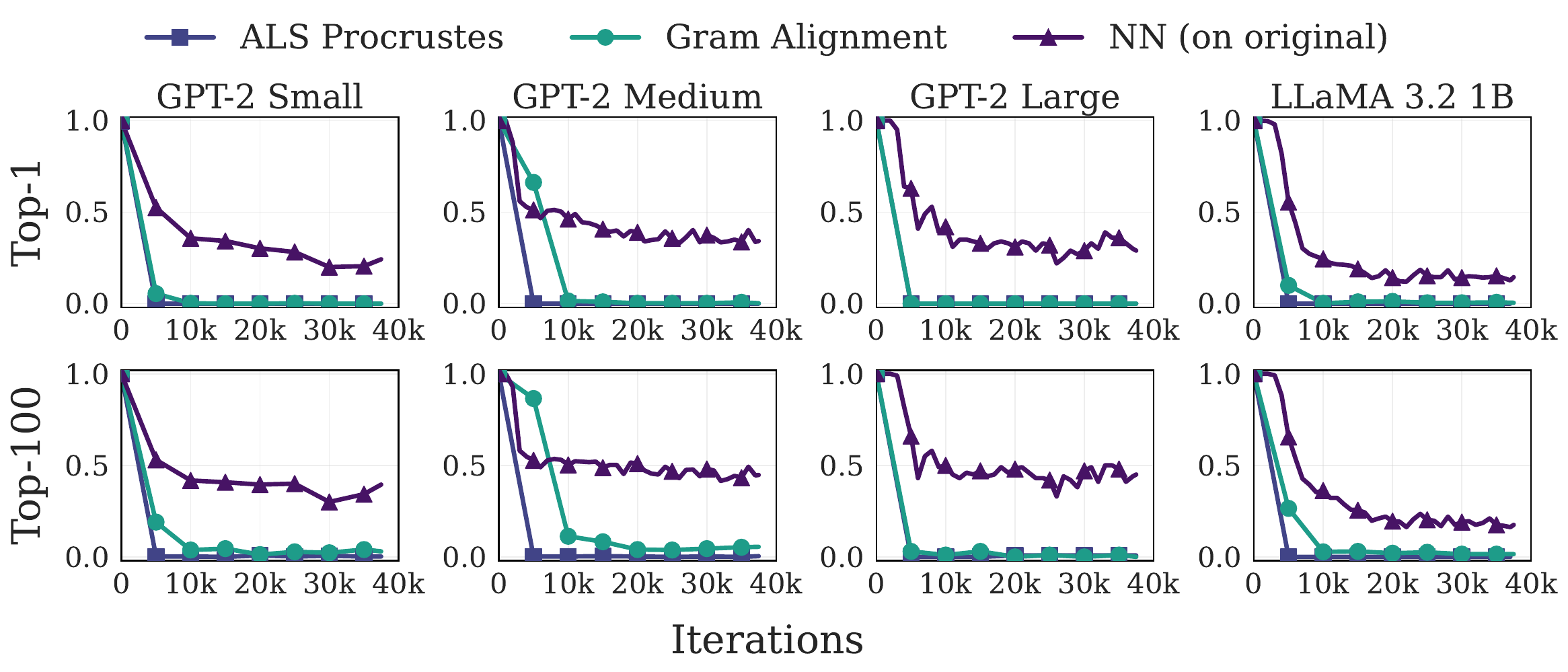}
    \caption{
    Top-1 and top-100 token-recovery accuracy for different attacks under fine-tuning.
    NN inversion is evaluated on nonobfuscated models, while ALS--Procrustes and Gram-based weight alignment attacks are evaluated on obfuscated models.
    }
    \label{fig:attack-topk-curves}
\end{figure*}

\begin{table*}[t]
\caption{Token-recovery accuracy (\%) evaluated at the checkpoint with minimum validation perplexity. Validation perplexities for Original (Orig) and \method (Ours) are shown in the header. \NA{} indicates that an attack is not applicable in that setting.}
\label{tab:attacks-summary}
\centering
\begin{small}
\setlength{\tabcolsep}{4.5pt}
\renewcommand{\arraystretch}{1.05}
\begin{tabular}{lcccccccc}
\toprule
& \multicolumn{2}{c}{GPT-2 S} & \multicolumn{2}{c}{GPT-2 M} & \multicolumn{2}{c}{GPT-2 L} & \multicolumn{2}{c}{Llama~3.2~1B} \\
\cmidrule(lr){2-3}\cmidrule(lr){4-5}\cmidrule(lr){6-7}\cmidrule(lr){8-9}
Attack & Orig & Ours & Orig & Ours & Orig & Ours & Orig & Ours \\
\midrule
Val PPL (min) & 12.71 & 12.70 & 12.02 & 12.07 & 12.51 & 12.46 & 13.77 & 13.62 \\
\midrule

\multicolumn{9}{c}{\textbf{Top-1}} \\
\midrule
NN inversion      & 24.3 & \NA & 34.3 & \NA & 29.0 & \NA & 14.3 & \NA \\
Norm-based        & 0.0 & 0.0 & 0.0 & 0.0 & 0.0 & 0.0 & 0.0 & 0.0 \\
ALS--Procrustes   & \NA & 0.0 & \NA & 0.0 & \NA & 0.0 & \NA & 0.0 \\
Gram Alignment    & \NA & 0.0 & \NA & 0.3 & \NA & 0.0 & \NA & 1.3 \\
\midrule

\multicolumn{9}{c}{\textbf{Top-10}} \\
\midrule
NN inversion      & 35.8 & \NA & 39.3 & \NA & 38.0 & \NA & 16.8 & \NA \\
Norm-based        & 0.0 & 0.0 & 0.0 & 0.0 & 0.0 & 0.0 & 0.0 & 0.0 \\
ALS--Procrustes   & \NA & 0.0 & \NA & 0.0 & \NA & 0.0 & \NA & 0.0 \\
Gram Alignment    & \NA & 1.0 & \NA & 1.0 & \NA & 0.0 & \NA & 1.3 \\
\midrule

\multicolumn{9}{c}{\textbf{Top-100}} \\
\midrule
NN inversion      & 39.5 & \NA & 44.8 & \NA & 47.0 & \NA & 19.0 & \NA \\
Norm-based        & 0.0 & 0.0 & 0.0 & 0.0 & 0.0 & 0.0 & 0.0 & 0.0 \\
ALS--Procrustes   & \NA & 0.2 & \NA & 0.4 & \NA & 0.0 & \NA & 0.0 \\
Gram Alignment    & \NA & 3.0 & \NA & 5.5 & \NA & 0.8 & \NA & 2.0 \\
\bottomrule
\end{tabular}
\end{small}
\end{table*}

\subsection{Token Identification by Attention Fingerprints}\label{app:statmatch}

As discussed in Section~\ref{subsec:statistical-matching}, attention logits remain invariant under \method.
An attacker can therefore use them as fingerprints for statistical matching.

In this appendix, we show empirically that additional fine-tuning breaks the association between these fingerprints and the underlying tokens.

When a user sends a sequence of length $L$, the server can store the inputs and perform one-token inference offline.
Since modern LLMs typically do not use absolute positional embeddings, a token's embedding does not depend on its position after the client sends it.
In particular, for a given token $i$, the server can compute all self-attention logits across all layers and heads.

For each token \(i\), the attacker can construct a feature vector
\[
f(i) \;=\; \bigl(\ell_{1,1}(i), \dots, \ell_{B,H}(i)\bigr) \in \mathbb{R}^{BH},
\]
where \(\ell_{b,h}(i)\) denotes the single-position self-attention logit at layer \(b\) and head \(h\).

Changing only the secret orthogonal matrices does not change the vector $f(i)$, since attention logits are invariant under the corresponding basis transformation.
As a result, the server can accumulate statistics over time, which may weaken the defense.

We empirically show that fine-tuning mitigates this issue by shifting the vectors $f(i)$ away from their previous values, thereby breaking the link between newly observed tokens and their accumulated statistics.
To evaluate this effect, we perform top-$k$ nearest-neighbor search in the space $\mathbb{R}^{BH}$.

Using the same fine-tuning protocol as in the main experiments, we observe that top-$k$ matching accuracy drops substantially, introducing significant uncertainty into statistical matching as shown in Figure~\ref{fig:statistical-matching}.

\begin{figure*}[t]
    \centering
    \includegraphics[width=\textwidth]{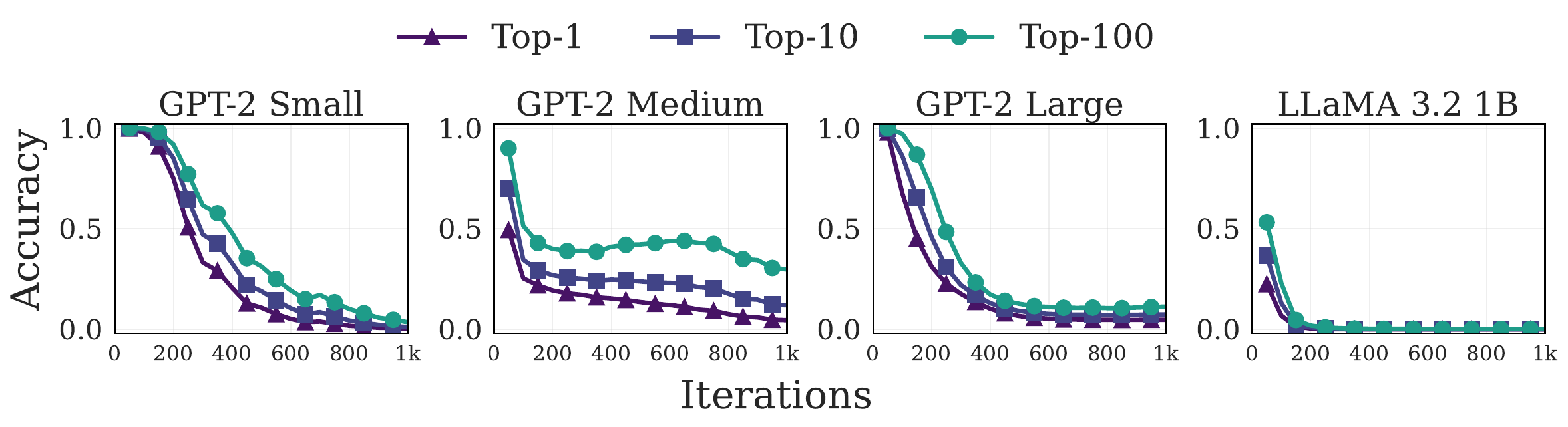}
    \caption{
    Token recovery using invariant attention-logit fingerprints.
    Top-1, top-10, and top-100 accuracy as a function of fine-tuning iterations for GPT-2 Small, Medium, Large, and Llama~3.2~1B.
    }
    \label{fig:statistical-matching}
\end{figure*}

\end{document}

%% file: math_commands.tex
\usepackage{amsmath,amsfonts,bm}

\def\eqref#1{(\ref{#1})}

\def\1{\bm{1}}

\DeclareMathAlphabet{\mathsfit}{\encodingdefault}{\sfdefault}{m}{sl}
\SetMathAlphabet{\mathsfit}{bold}{\encodingdefault}{\sfdefault}{bx}{n}



%% file: macros.tex
\usepackage{amsmath}
\usepackage{amssymb}
\usepackage{mathtools}
\usepackage{amsthm}
\usepackage{nicefrac}
\usepackage{multirow}
\usepackage{tikz-cd} 
\usepackage{subcaption}
\usepackage{multicol,caption}

\usepackage{wrapfig}
\usepackage{algorithm}

\usepackage[flushleft]{threeparttable} 
\usepackage{colortbl}
\definecolor{bgcolor}{rgb}{0.8,1,1}
\definecolor{bgcolor2}{rgb}{0.8,1,0.8}
\definecolor{niceblue}{rgb}{0.0,0.19,0.56}

\usepackage{hyperref}
\hypersetup{colorlinks,linkcolor={blue},citecolor={blue},urlcolor={blue}}

\usepackage{pifont}
\definecolor{PineGreen}{RGB}{0,110,51}
\definecolor{BrickRed}{RGB}{143,20,2}

\usepackage{thmtools}

\usepackage[colorinlistoftodos,bordercolor=orange,backgroundcolor=orange!20,linecolor=orange,textsize=scriptsize]{todonotes}

\definecolor{shadecolor}{gray}{0.9}
\declaretheoremstyle[
headfont=\normalfont\bfseries,
notefont=\mdseries, notebraces={(}{)},
bodyfont=\normalfont,
postheadspace=0.5em,
spaceabove=1pt,
mdframed={
  skipabove=8pt,
  skipbelow=8pt,
  hidealllines=true,
  backgroundcolor={shadecolor},
  innerleftmargin=4pt,
  innerrightmargin=4pt,
  innertopmargin=10pt,    
  innerbottommargin=6pt
}
]{shaded}

\declaretheorem[style=shaded,within=section]{definition}

\declaretheorem[style=shaded,sibling=definition]{lemma}

\usepackage[capitalize,noabbrev]{cleveref}

\usepackage{xspace}
\usepackage{xspace}